\documentclass[sigconf]{acmart}

\AtBeginDocument{%
  \providecommand\BibTeX{{%
    \normalfont B\kern-0.5em{\scshape i\kern-0.25em b}\kern-0.8em\TeX}}}

\setcopyright{acmcopyright}
\copyrightyear{2018}
\acmYear{2018}
\acmDOI{10.1145/1122445.1122456}

\acmConference[Woodstock '18]{Woodstock '18: ACM Symposium on Neural
  Gaze Detection}{June 03--05, 2018}{Woodstock, NY}
\acmBooktitle{Woodstock '18: ACM Symposium on Neural Gaze Detection,
  June 03--05, 2018, Woodstock, NY}
\acmPrice{15.00}
\acmISBN{978-1-4503-XXXX-X/18/06}

\usepackage{graphicx}
\usepackage{caption}
\usepackage{subcaption}

\usepackage{amsmath, amsthm, mathtools,stmaryrd}

\usepackage{bbm}
\usepackage{algorithm, algorithmic}

\usepackage{enumitem}
\usepackage{sidecap}

\usepackage{url}            
\usepackage{booktabs}       
\usepackage{nicefrac}       
\usepackage{braket}

\newtheorem{theorem}{Theorem}[section]
\newtheorem*{theorem*}{Theorem}

\newtheorem*{proposition*}{Proposition}

\newtheorem*{lemma*}{Lemma}

\newtheorem*{conjecture*}{Conjecture}

\newtheorem*{fact*}{Fact}

\newtheorem*{hypothesis*}{Hypothesis}

\newtheorem*{claim*}{Claim}

\theoremstyle{definition}
\newtheorem{definition}[theorem]{Definition}

\theoremstyle{remark}

\newtheorem*{remark*}{Remark}

\newtheorem*{observation*}{Observation}

\newcommand{\gnorm}[1]{{\vert\kern-0.25ex\vert\kern-0.25ex\vert #1 
    \vert\kern-0.25ex\vert\kern-0.25ex\vert}}

\newcommand{\Psymb}{\mathbb{P}}

\DeclareMathOperator*{\ProbOp}{\Psymb}

\renewcommand{\Pr}{\ProbOp}

\newcommand{\eps}{\varepsilon}
\renewcommand{\epsilon}{\varepsilon}

\newcommand{\epsobs}{\gamma}
\newcommand{\epsun}{\eps}

\setcopyright{rightsretained}
\begin{document}

    \title{Measuring Model Fairness under Noisy Covariates: A Theoretical Perspective}

\author {
    Flavien Prost,
    Pranjal Awasthi,
}
\authornote{The first two authors contributed equally.}
\author {
    Nick Blumm,
    Aditee Kumthekar,
    Trevor Potter \\
    Li Wei,
    Xuezhi Wang,
    Ed H. Chi,
    Jilin Chen,
    Alex Beutel
}
\affiliation{
  \institution{Google}
  \country{USA}
}

\renewcommand{\shortauthors}{Prost et al.}





\begin{abstract}
In this work we study the problem of measuring the fairness of a machine learning model under noisy information. Focusing on group fairness metrics, we investigate the particular but common situation when the evaluation requires controlling for the confounding effect of covariate variables. In a practical setting, we might not be able to jointly observe the covariate and group information, and a standard workaround is to then use proxies for one or more of these variables. Prior works have demonstrated the challenges with using a proxy for sensitive attributes, and strong independence assumptions are needed to provide guarantees on the accuracy of the noisy estimates. In contrast, in this work we study using a proxy for the covariate variable and present a theoretical analysis that aims to characterize weaker conditions under which accurate fairness evaluation is possible.
Furthermore, our theory identifies potential sources of errors and decouples them into two interpretable parts $\epsobs$ and $\epsun$. The first part $\epsobs$ depends solely on the performance of the proxy such as precision and recall, whereas the second part $\epsun$ captures correlations between all the variables of interest.
We show that in many scenarios the error in the estimates is dominated by $\epsobs$ via a linear dependence, whereas the dependence on the correlations $\epsun$ only constitutes a lower order term. 
As a result we expand the understanding of scenarios where measuring model fairness via proxies can be an effective approach. Finally, we compare, via simulations, the theoretical upper-bounds to the distribution of simulated estimation errors and show that assuming some structure on the data, even weak, is key to significantly improve both theoretical guarantees and empirical results.

\end{abstract}

\copyrightyear{2021}
\acmYear{2021}
\acmConference[AIES '21]{Proceedings of the 2021 AAAI/ACM Conference on AI, Ethics, and Society}{May 19--21, 2021}{Virtual Event, USA}
\acmBooktitle{Proceedings of the 2021 AAAI/ACM Conference on AI, Ethics, and Society (AIES '21), May 19--21, 2021, Virtual Event, USA}\acmDOI{10.1145/3461702.3462603}
\acmISBN{978-1-4503-8473-5/21/05}



\maketitle

\section{Introduction}
\label{sec:intro}
As machine learning~(ML) systems permeate everyday life it is of utmost importance to understand and correct for the underlying biases present in such systems. As a result there has been a significant interest in recent years in the area of {\em algorithmic fairness}~\citep{kleinberg2017, chouldechova2017fair, dieterich2016}. An important and often overlooked aspect of addressing biases in AI systems is the challenge of how to actually compute the fairness metric in practice. Given an ML model, fairness metrics typically involve evaluating the disparity in the model performance across different demographics and/or data slices of interest, according to a well defined metric such as {\em false positive rate}, {\em calibration} etc. However, often these criteria cannot be exactly evaluated as they require information that is only sparsely available in practical applications. How should one approach measuring fairness metrics in such situations? Our goal is to address this question from a theoretical perspective.

We focus on scenarios that involve measuring group fairness criteria where a practitioner wants to account for a covariate that is not available jointly with the group information. This challenge arises in a variety of applications.  For example, consider the case of evaluating a model predicting the toxicity of an online comment, where we want to measure differences in performance of the model across toxic comments directed at different demographic groups, but we only want to compare performance within the same topic or language \citep{Dixon,DBLP:journals/corr/HosseiniKZP17, mubarak-etal-2017-abusive, park-etal-2018-reducing}. In order to measure the performance of the model, one would ideally want access to a corpus of comments, each containing information about the demographic group, topic or language, and the prediction of the model for that comment. However, in practice such joint observations are rarely available. 
In this case, one would frequently use a model to predict the covariate~(topic or language) for each comment, but might wonder how inaccuracies in this model of topic or language influence our measurement of a bias in the toxicity prediction.
Another common occurrence is in the context of a recommender system, where a model predicts whether a user would click an item, but a practitioner might want to only consider fairness over high quality items, such as avoiding clickbait or focusing on when the user would actually be satisfied with the item~\citep{akoglu2010oddball, horne2017just, kumar2018false, WatchNext}. Again, it is unlikely that a dataset contains together clicks, item group information, and quality or post-click engagement information; can a practitioner use a model of item quality to accurately estimate the bias in the recommender? Across all of these examples it's important to understand the confidence in our measurements.

Group fairness has received significant attention in recent years resulting in a variety of natural fairness criteria~\citep{barocas-hardt-narayanan} that machine learning systems should satisfy. 
Our work concerns two common variants, {\em Statistical Parity}~\citep{dwork2012fairness} and {\em Equality of opportunity}~\citep{hardt2016equality}, which measure the difference of positive rates (or true positive rates) across sub-populations. To model the presence of the confounding or mediating covariate as described in the applications above, one can extend fairness metrics by conditioning on this covariate and this approach has been used in a variety of prior works~\citep{ritov2017conditional, chouldechova2017fair, kilbertus2017avoiding, hebert2017calibration, kearns2018preventing, beutel2019putting}. 

We summarize our targeted fairness metric as follows. Consider a machine learning system making a prediction $y \in \{0,1\}$, a sensitive attribute $\ell \in \{0,1\}$ and a covariate $v \in \{0,1\}$. The {\em Conditional Statistical Parity}~\citep{Corbett17} metric is defined as: 
\begin{align}
\label{eq:conditional-SP}
G_{SP} \!=\! \Pr [y\!=\!1 | v\!=\!1, \ell\!=\!0] \!-\! \Pr[y\!=\!1 |v\!=\!1, \ell\!=\!1 ].     
\end{align}
Similarly, let $y^* \in \{0,1\}$ be the ground truth label. Then the {\em Conditional Equal Opportunity} metric is defined as:
\begin{align}
\label{eq:conditional-EO}
G_{EO} = &\Pr[y=1 | y^*=1, v=1, \ell=0] \nonumber \\
&\;\;\;\;- \Pr[y=1 | y^*=1, v=1, \ell=1].
\end{align}
For ease of exposition and reducing notation, the rest of the paper focuses on the metric defined in \eqref{eq:conditional-SP}. \emph{However, all of our theoretical results trivially generalize to the case of conditional equal opportunity as well by focusing on sub-populations where $y^*=1$.} 

While it is important to incorporate legitimate covariates in fairness metrics, evaluating the metrics accurately requires a large dataset that contains both the group label $\ell$, the covariate information $v$ and possibly the label $y^*$~(in the case of equal opportunity). As we discussed before, in practice we might only have an estimate $\hat{v}$ derived from either a proxy attribute in the data that is correlated with $v$, or a classifier trained to predict $v$. A natural question that comes up is whether one can output 
\begin{align}
\hat{G}_{SP} \!=\! \Pr[y\!=\!1 | \hat{v} \!=\! 1, \ell\!=\!0] \!-\! \Pr[y\!=\!1 | \hat{v} \!=\! 1, \ell\!=\!1]
\end{align} 
as an estimate of the bias $G_{SP}$ of the model as in \eqref{eq:conditional-SP}. 

For the sake of clarity, let's return to our example from above of a toxicity model for online comments, which predicts whether a comment contains toxic language ($y=1$ if toxic) and for which prior works have demonstrated that these models often exhibit biases towards certain demographics~\citep{Dixon, DBLP:journals/corr/HosseiniKZP17, mubarak-etal-2017-abusive, park-etal-2018-reducing}. At the same time, the level of toxicity might be significantly influenced by the topic of the comment (i.e., topic is covariate $v$) and we may want to control for its impact in our fairness analysis. Following the Statistical Parity approach, a reasonable fairness metric would be given by Eq. \eqref{eq:conditional-SP}\footnote{Note, in practice one might prefer to follow the equality of opportunity metric for this or other applications.  We use the conditional statistical parity metric throughout the paper for simplicity of notation, but our results easily extend to the conditional equality of opportunity metric as well, and we encourage practitioners to always consider which metric fits best with their application.}. 
As both topic and demographic information are typically available for only a small portion of examples, a practitioner might not have access to the joint distribution of ($v$, $l$) but a topic classifier $\hat{v}$ might be available. In that situation, how confident can a practitioner be to use $\hat{G}_{SP}$ as an estimation of its intended metric $G_{SP}$?

In the rest of the paper we study conditions under which we can bound the estimation error $|G_{SP} - \hat{G}_{SP}|$. In Section~\ref{sec:simple-cases}, we highlight through simple examples two different sources that will affect the estimation error: (a) the performance of the proxy~($\hat{v}$) and, (b) the joint correlations between variables ($y, v$ and $\hat{v}$) that can be represented as the value of the outcome value $y$ over the confusion matrix of ($v, \hat{v}$). The bulk of the theory aims to derive error bounds based on the extent to which the data distribution satisfies certain conditions that depend on the above two sources. Throughout the paper, parameters related to the classifier performance will be captured by $\epsobs$, and parameters related to the joint correlations between variables $y,v$ and $\hat{v}$ will be captured by $\epsun$. These parameters will always take values in $[0,1]$.

In Section~\ref{sec:bound-classifier}, we derive a bound on the estimation error based solely on the precision and recall of the proxy $\hat{v}$. If $1-\epsobs_A$ is the lower bound on the worst case precision or recall of the proxy on any data slice, i.e., conditioned on $\ell=0$ or $\ell=1$, then we show that the error $|G_{SP} - \hat{G}_{SP}|$ can be bounded by $2 \cdot \epsobs_A$. However, $2 \cdot \epsobs_A$ may not always be small enough for fairness evaluation and our goal will be to progressively refine this bound by exploiting other correlations between variables ($y, v, \hat{v}$ and $\ell$).

In Section~\ref{sec:bound-graph}, we identify two cases depending on correlations  which might lead to ``errors canceling each other''. In the first case (B1), if the precision and recall of $\hat{v}$ are $\epsobs_{B1}$-close to each other (for $\ell=0,1$), then we show that the error is bounded by a linear combination of two terms $\epsobs_{B1}$ and $\epsun_{B1}$. The term $\epsun_{B1}$ captures how the variables of interest, i.e., $y, v$ and $\hat{v}$ correlated within a data slice~($\ell=0$ or $\ell=1$). 
In the second case (B2), we study how the error behaves as a function of closeness of precision across slices ($\ell= 0$ vs $\ell=1$), and recall across slices. In this case we again show that the error depends on a linear combination of two terms $\epsobs_{B2}$ and $\epsun_{B2}$, one capturing relative closeness of precision and recall, and the other how much the correlations between $y, v$ and $\hat{v}$ vary between $\ell=0$ and $\ell=1$.

Having identified the various sources of errors above, in Section~\ref{sec:combining-ABC}, we provide a refined bound that captures how the various errors interact with each other. As a result we show in Theorem~\ref{thm:combine-ABC} that the final estimation error depends linearly in the $\epsobs$ parameters and quadratically in $\epsun$ parameters. Hence, the various errors due to correlations have a multiplying effect leading to a lower order term. Our analysis thereby reveals that the properties of the proxy classifier such as precision and recall play a more important role in the final estimation error, but that the structure of the correlations through $\epsun$ can help reduce the bound on $|G - \hat{G}|$.

Finally, we use a simulation framework to compare our theoretical bounds on estimation errors to their simulated counterparts. We show that our worst-case bounds are met when no assumption is imposed on the data generation process. Furthermore, we show how enforcing some (even weak) structure on the correlations~(through the parameter $\epsun$) can significantly improve both theoretical bounds and empirical distributions, compared to when we rely only on precision and recall.

Our work is similar in spirit to prior works on studying the effect of label noise~(outcome variable $y^*$) on fairness metrics~\citep{fogliato2020fairness, jiang2020identifying}, as well as noise in the sensitive attribute $\ell$\citep{awasthi2021faact, chen2019fairness, kallus2020}. The case of label noise is captured by our theoretical analysis as one can simply take the noisy covariate to be $y^*$. The case of noise in the sensitive attributes is not directly comparable to our setting. We discuss this more in Section~\ref{sec:related}.

The rest of the paper is structured as follows. In Section~\ref{sec:related} and Section~\ref{sec:notation} we review related work and discuss preliminaries. In Section~\ref{sec:simple-cases} we discuss several simple cases where accurate bias estimation is possible provided certain independence assumptions hold. This section is meant to serve as a warm up to help the reader ease into the notation.
In Section~\ref{sec:bound-classifier}, we build our main theory by first bounding the estimation error solely based on the performance of the $\hat{v}$ classifier, and next identifying, in Section~\ref{sec:bound-graph}, two other conditions that affect the estimation error. We present our main theorem in Section~\ref{sec:combining-ABC} that shows that when all the conditions hold true to different extents, the estimation error can be decomposed into a linear part that captures the properties of the proxy and a quadratic part that captures the correlations. We conclude with experiments on simulated data in Section~\ref{sec:experiments} and discuss conclusions in Section~\ref{sec:conclusions}.

\textbf{Note:} We would like to point out that in this work we are focused on understanding the theoretical underpinnings of using proxies in evaluation, and \emph{do not} advocate always using the approach of measuring the bias of model via a proxy attribute. As pointed out in prior works~\citep{awasthi2021faact} the approach could have unintended side effects on the evaluation, and as such practitioners should carefully consider the risks of using such an approach.  Our goal is to identify theoretically how errors propagate in the analysis to affect the final outcome. We hope that our work will help practitioners make more informed choices.

\section{Related Work}
\label{sec:related}
We focus on conditional metrics such as conditional statistical parity and conditional equal opportunity. Several recent works have argued the use of such metrics over or in addition to their un-conditioned counterparts. For example, the works of~\citet{ritov2017conditional, Corbett17,chouldechova2017fair} all explore conditioning on the number of prior convictions when measuring the bias of risk assessment tools for recidivism.
This naturally leads to the conditional statistical parity and equality of opportunity metrics. 
The work of~\citet{beutel2019putting} proposes the use of the conditional equal opportunity metric to account for differences in label confidence. 
Recent works on intersectional fairness consider metrics that involve conditioning on several covariates or arbitrary data slices described by functions of low VC-dimension~\citep{kearns2018, hebert2017calibration}.

The problem of assessing unfair biases and training of fair machine learning models for the above discussed conditional metrics, under noisy label or covariate information has started to receive attention from the research community in recent years. Here we discuss the works most relevant in the context of our results. 

\paragraph{Evaluating under uncertainty}
The work of~\citet{chen_fairness_under_unawareness} studies estimating the {\em mean demographic disparity} when one has noisy information about the sensitive attribute. They analyze proxies based on threshold based estimators and present conditions under which the method over-estimates or under-estimates the true metric. However, unless strong independence assumptions are made there is no guarantee that the estimation error will be small. The work of~\citet{kallus2020} extends this to other fairness metrics such as equality of opportunity and presents methods to construct uncertainty intervals around the true value of the metric of interest. Again, excluding certain independence assumptions the provided intervals can be large. 

The recent work of~\citet{awasthi2021faact} studies what properties of the classifier used to construct the proxy affect the estimation error. They point out that accuracy of the classifier alone is not enough and advocate the use of active learning based algorithms to deal with uncertain information. We would like to point out that the above mentioned works consider noise in the sensitive attribute and are not directly comparable to our setting of noisy covariates. At a technical level, when the noisy variable $\ell$ is also the one for which we want to access performance gap, i.e., $\ell=0$ vs. $\ell=1$, it is not easy to decouple the various sources of error as we do in the case of noisy covariates. As a result we find that in our setting one can get reasonably small bounds even without strong independence assumptions.
 
The recent work of~\citet{fogliato2020fairness} studies how noise in the observed labels~(the outcome variable) affects estimates of fairness metrics such as false positive rates, false negative rates, and positive predicted value. In principle the setting of label noise can be captured by our theory. However, the specific assumptions used make our work not directly comparable with that of~\citet{fogliato2020fairness}. We provide error bounds that depend on both the precision and recall, as well as other correlations between the variables. The authors in~\citet{fogliato2020fairness} only consider how the estimation error depends on the precision~(i.e., the noise rate). As a result they need to make strong assumptions such as one-sided noise~($y^*=1$), and noise affecting only one demographic group. 
We show on the other hand that by modeling other natural quantities of interest such as recall and correlations, one can avoid independence assumptions and at the same time obtain estimation error guarantees.

\paragraph{Training under uncertainty}
Beyond measuring fairness metrics, a recent line of work has also studied the problem of training fair classifiers under noise in labels and/or sensitive attributes. The work of~\citet{jiang2020identifying} proposes a method to train a fair classifier via a re-weighting scheme given access to noisy labels. The works of~\citet{zafar2017,lamy2019, gupta2018} show how to incorporate fairness constraints during training when there is noise in the sensitive attributes. Relatedly in training approaches, \citet{Hashimoto2018,lahoti2020fairness} study the extreme scenario where sensitive attributes are entirely unavailable and propose re-weighting approaches to boost low-performing groups; \citet{coston2019} apply transfer learning techniques to overcome missing sensitive attributes. Further recent work studies the tradeoff between fairness and accuracy and how it is affected by label noise \citep{wick2019unlocking}. 

\section{Notation and Preliminaries}
\label{sec:notation}
As stated before, we will focus on the conditional statistical parity metric defined in Eq.~\eqref{eq:conditional-SP}. However all our results continue to hold for the conditional equal opportunity metric, Eq.~\eqref{eq:conditional-EO},  as well. Let $\ell,v,y$ denote $\{0,1\}$-valued random variables where $\ell$ is the sensitive attribute, $v$ is the relevant covariate and $y$ is the predicted outcome. Our goal is to approximate
\begin{align}
    \label{eq:ideal-bias}
    G & \coloneqq \Pr[y\!=\!1 | v\!=\!1, \ell\!=\!1] - \Pr[y\!=\!1 | v\!=\!1, \ell\!=\!0].
\end{align}
by computing
\begin{align}
    \label{eq:computed-bias}
    \hat{G} & \coloneqq \Pr[y\!=\!1 | \hat{v}\!=\!1, \ell\!=\!1] - \Pr[y\!=\!1 | \hat{v}\!=\!1, \ell\!=\!0].
\end{align}
We want to characterize conditions under which the error in our estimates, i.e., $|G - \hat{G}|$ can be bounded. We introduce below notations for two important parameters of our system.

\noindent \textbf{Classifier Performance.} We define the conditional precision and recall of the proxy $\hat{v}$ that will play an important role in our analysis.
\begin{align}
\label{def:cond-precision-recall}
    \Pr[{v}\!=\!1 | \hat{v}\!=\!1, \ell\!=\!l] &= 1-p_l\\
    \Pr[\hat{v}\!=\!1 | v\!=\!1, \ell\!=\!l] &= 1-r_l.
\end{align}
Constraints on the classifier perfomance will be captured by the variable $\epsobs$.

\noindent \textbf{Outcome $y$ over the confusion matrix of $\hat{v}$.}
We will also study how correlations among the variables $y,v,\hat{v}$ affect the error in the measurement of $G$. In later sections we will quantify these correlations in terms of the value of the outcome $y$ over the confusion matrix of ($v,\hat{v}$) conditioned on $\ell=0,1$ that we define in Table~\ref{tbl:confusion-matrix} for group $\ell=l$. For the sake of brevity we will overload notation and use the term confusion matrix to refer to the Table above. Constraints on the confusion matrix will be captured by the variable $\epsun$.

\begin{table}
\centering
\begin{tabular}{|c|c|}
\hline
 $\Pr[y\!=\!1 | v\!=\!0, \hat{v}\!=\!1, \ell\!=\!l]$ & $\Pr[y\!=\!1 | v\!=\!1, \hat{v}\!=\!1, \ell\!=\!l]$\\
 \hline
 $\Pr[y\!=\!1 | v\!=\!0, \hat{v}\!=\!0, \ell\!=\!l]$ & $\Pr[y\!=\!1 | v\!=\!1, \hat{v}\!=\!0, \ell\!=\!l]$\\
 \hline
\end{tabular}
    \caption{\label{tbl:confusion-matrix} Value of the outcome $y$ over the confusion matrix of $v,\hat{v}$, conditioned on group $\ell=l$.}
\end{table}

\section{Simple cases}
\label{sec:simple-cases}
Before proceeding to our main analysis in the next sections, we state a few simple scenarios where accurate bias estimation via $\hat{v}$ is possible. The proofs from this section can be found in the Appendix.

\subsection{Case 1: $y \perp \{v,\hat{v}\} | \ell$.}  

This condition implies that $y$ is independent of the covariate~($v$) and the proxy~($\hat{v}$) when conditioned on $\ell$. If the variables $y,v$ and $\hat{v}$ are viewed as nodes in a graphical model~\citep{wainwright2008graphical}, then the above assumption encodes the graph structure in Figure~\ref{fig:gm2}. 
\begin{figure}[htbp]
    \centering
    \includegraphics[width=0.3\textwidth]{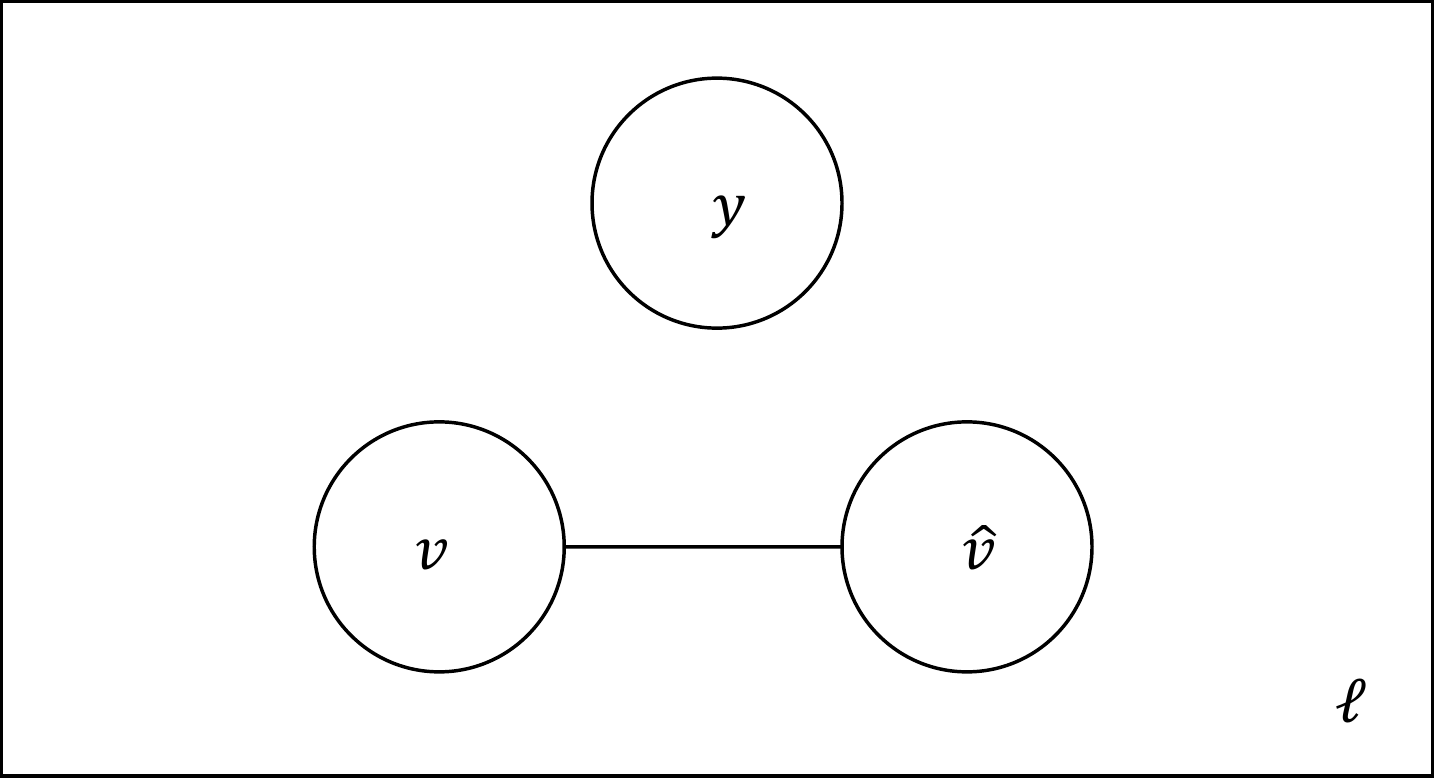}
    \caption{A graphical model over three random variables under the assumption that $y$ is independent of $\hat{v}$ and $v$. All the events are conditioned on $\ell$.}
    \label{fig:gm2}
\end{figure}
In this case we prove the following theorem:
\begin{theorem}
\label{thm:strong-independence}
If the independence assumption encoded in Figure~\ref{fig:gm2} holds, then
$$
G  = \hat{G}
$$
\end{theorem}

The theorem is intuitive as the independence assumption makes conditioning on either $v$ or $\hat{v}$ irrelevant. A scenario where the assumption will hold is when the outcome $y$ is completely determined by the value of $\ell$, the sensitive attribute. This is typically too strong a condition to hold in practice.

\subsection{Case 2: $y \perp \hat{v} | v,\ell$.} 
Next we consider our first non-trivial case where $y$ is independent of the proxy $\hat{v}$ when conditioned on $\ell$ and $v$. This assumption is encoded in Figure~\ref{fig:gm3}.
\begin{figure}[htbp]
    \centering
    \includegraphics[width=0.3\textwidth]{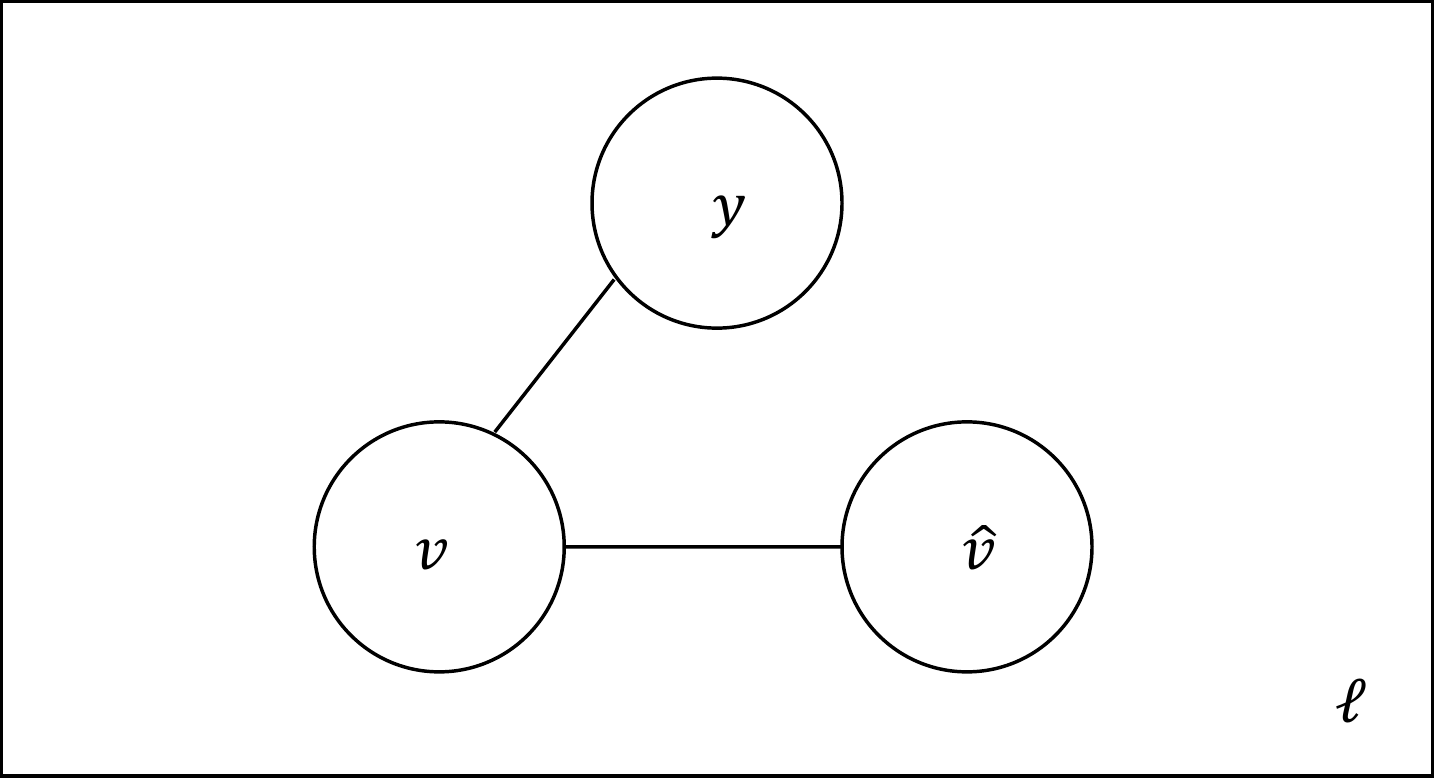}
    \caption{A graphical model over three random variables under the assumption that $y$ is independent of $\hat{v}$ conditioned on $v$. All the events are conditioned on $\ell$.}
    \label{fig:gm3}
\end{figure}

We will show that if the proxy $\hat{v}$ has high conditional precision $p_l$, then we will get good estimates via $\hat{G}$.

\begin{theorem}
\label{thm:high-precision}
If the independence assumption encoded in Figure~\ref{fig:gm3} holds, and for any $\ell \in \{0,1\}$ the conditional precision of the proxy $\hat{v}$ is at least $1-p$, then we have that
$$
|G - \hat{G}| \leq 2p.
$$
\end{theorem}

It is interesting to note that the bound does not depend on the recall of the proxy $\hat{v}$. To get an intuition behind this, notice that the errors in the estimate of $G$ involve both {\em false positives} ($v=0, \hat{v}=1$) and {\em false negatives} ($v=1, \hat{v}=0$). Indeed, we can write each term as:

\begin{align}
\begin{split}
\label{eq:intuition1}
\Pr[y\!=\!1 | v\!=\!1, \ell\!=\!0] &= (1-r_0) \Pr[y\!=\!1 | \hat{v}\!=\!1, v\!=\!1, \ell\!=\!0] \\
    &\;\;\;\; + r_0 \Pr[y\!=\!1 | \hat{v}\!=\!0, v\!=\!1, \ell\!=\!0] 
\end{split}
\end{align}

\begin{align}
\begin{split}
\label{eq:intuition2}
\Pr[y\!=\!1 | \hat{v}\!=\!1, \ell\!=\!0] &= (1-p_0) \Pr[y\!=\!1 | \hat{v}\!=\!1, v\!=\!1, \ell\!=\!0]\\
    &\;\;\;\; + p_0 \Pr[y\!=\!1 | \hat{v}\!=\!1, v\!=\!0, \ell\!=\!0].
\end{split}
\end{align}

The estimation error will be small if the two terms are close (similarly for $\ell=1$).

Due to the independence assumption, we have that $\Pr[y=1 | v=1, \hat{v}=0, \ell=0] = \Pr[y=1 | v=1, \hat{v}=1, \ell=0]$.  Therefore the first equation simply becomes $\Pr[y=1 | \hat{v}=1, v=1, \ell=0]$ and the dependency on $r_0$ disappears. Only false positives ($v=0, \hat{v}=1$) dictate the errors which is controlled by the precision of the proxy $\hat{v}$. 

One may expect the above condition to hold if the classifier for $\hat{v}$ is trained on a different independent data set and hence its errors may not be correlated with $y$. As an extreme case, the condition would hold if for any data point, $\hat{v}$ has a $1-p$ chance of being correct, independently of other points.

\subsection{Case 3: $y \perp {v} | \hat{v},\ell$.} 
Finally, the third case that we consider in this section is the complementary case where $y$ is independent of ${v}$ when conditioned on $\ell$ and $\hat{v}$. This is encoded in Figure~\ref{fig:gm4}.
\begin{figure}[htbp]
    \centering
    \includegraphics[width=0.3\textwidth]{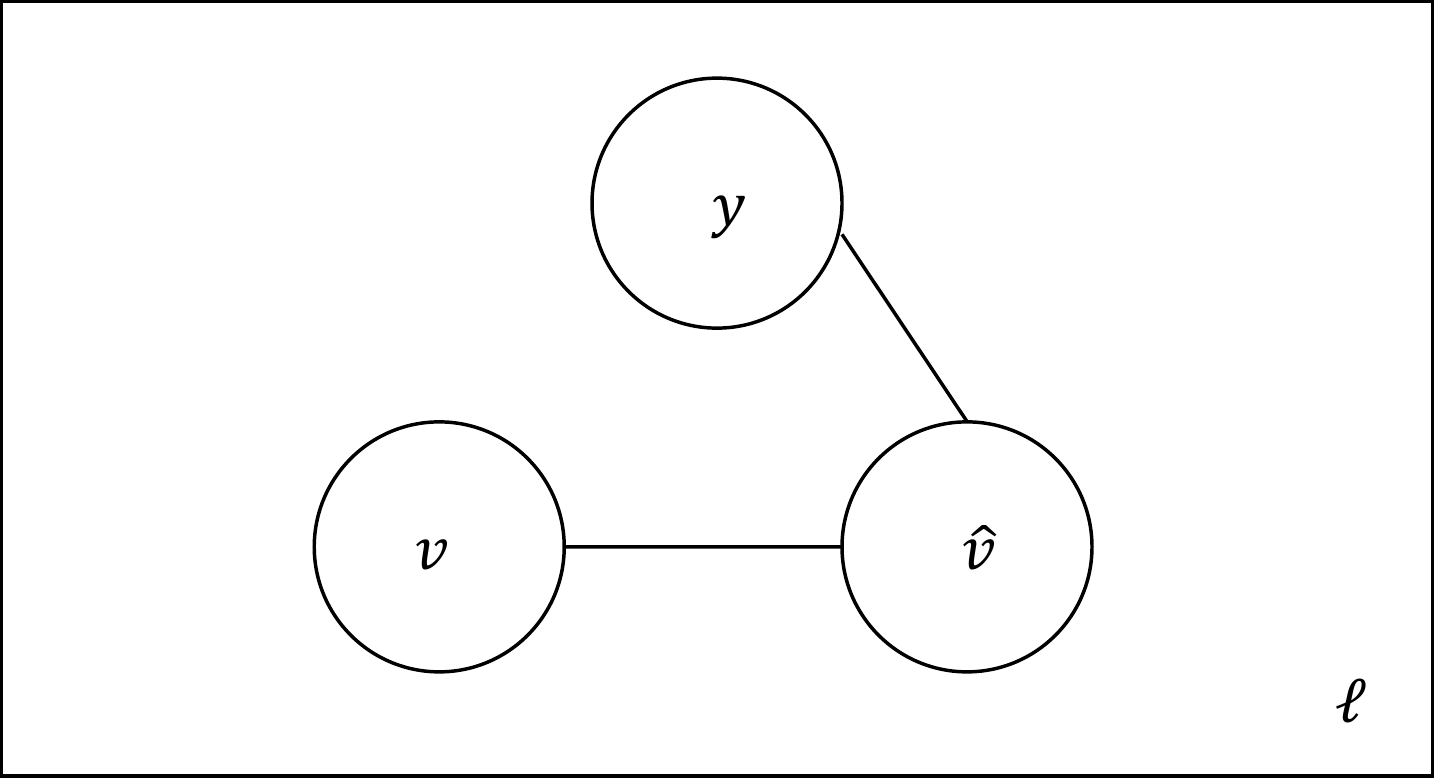}
    \caption{A graphical model over three random variables under the assumption that $y$ is independent of ${v}$ conditioned on $\hat{v}$. All the events are conditioned on $\ell$.}
    \label{fig:gm4}
\end{figure}
In this case we will show that if the proxy $\hat{v}$ has high recall~($\Pr[\hat{v}=1 | {v}=1, \ell]$) then we will get good estimates via $\hat{G}$. 
\begin{theorem}
\label{thm:high-recall}
If the independence assumption encoded in Figure~\ref{fig:gm4} holds, and for any $\ell \in \{0,1\}$ the conditional recall of the proxy $\hat{v}$ is at least $1-r$, then we have that:
$$
|G - \hat{G}| \leq 2r.
$$
\end{theorem}

Again, the bound above depends only on the recall and the intuition is similar to that of Theorem~\ref{thm:high-precision}. Analyzing again the equations~\eqref{eq:intuition1} and~\eqref{eq:intuition2}, we can see that~\eqref{eq:intuition2} is simply equal to $\Pr[y=1 | \hat{v}=1, v=1, \ell=0]$, with $p_0$ disappearing. Therefore only false negatives matter in this case which leads to the constraint on the recall. 

Similar to the previous case, one may expect the above condition to hold if the errors of $\hat{v}$ are somewhat random. For example, the condition will hold if every prediction of $\hat{v}$ has a $1-r$ chance of being correct, independently of other data points.

\subsection{Key take-aways.}
We presented in this section some simple cases relying on strong independence assumptions where the estimation error from using $\hat{v}$ can be bounded. The goal was to get the reader familiar with the notations but also to illustrate how the estimation error might be affected by two types of system parameters: the performance of the covariate proxy $\hat{v}$ (precision, recall) as well as the general correlations between parameters ($y, v, \hat{v}, \ell$). 
In the next sections, we present some natural conditions that do no rely on strong assumptions and yet guarantee accurate bias estimation when using $\hat{v}$. In Section \ref{sec:bound-classifier}, we will first bound the estimation error based only on the performance of the classifier $\hat{v}$. Then  in Section \ref{sec:bound-graph}, we will highlight two additional characteristics of the system in terms of ($y, v, \hat{v}, \ell$) that lead to smaller estimation errors. Finally in Section \ref{sec:combining-ABC}, we derive a refined bound that combines all the different considerations.

\section{Bounds based only on classifier performance}
\label{sec:bound-classifier}
Initially, a practitioner might be willing to only make some assumptions on the performance of the proxy $\hat{v}$. In this section, we will study this case and derive the associated bounds.

\paragraph{Main equations.} We first write some important equations that will be useful throughout the rest of the paper. We define
\begin{align}
\delta_0 &= \Pr[y\!=\!1 | v\!=\!1, \ell\!=\!0] - \Pr[y\!=\!1 | \hat{v}\!=\!1, \ell\!=\!0]
\end{align}
From Equation \eqref{eq:intuition1} and \eqref{eq:intuition2} we get that

\begin{align}
\label{eqn:equation_delta}
\begin{split}
\delta_0 &= (p_0 - r_0) \cdot (\Pr[y\!=\!1 | \hat{v}\!=\!1, v\!=\!1, \ell\!=\!0]) \\
         &\;\;\;\; + r_0 \Pr[y\!=\!1 | \hat{v}\!=\!0, v\!=\!1, \ell\!=\!0] \\
         &\;\;\;\; - p_0 \Pr[y\!=\!1 | \hat{v}\!=\!1, v\!=\!0, \ell\!=\!0]
 \end{split}
\end{align} 
In a similar manner we can define $\delta_1$ that is conditioned on the slice $\ell=1$. Eq. \eqref{eqn:equation_delta} represents the estimation error made on the slice $\ell=0$ and has a natural interpretation. The second term is the error generated by not including {\em false negative examples}~($\hat{v}=0, v=1$) into the $\hat{G}$ term, while the third term is the error caused by including incorrectly, the {false positive} examples~($\hat{v}=1, v=0$). The final estimation error is the difference between these two terms:
\begin{align}
\label{eqn:equation_g}
|G - \hat{G}| = |\delta_1 - \delta_0|
\end{align}
It is easy to see that the magnitude of each term in $\delta_0$ and $\delta_1$ depends on the performance of the proxy $\hat{v}$. We next investigate this dependence.

\subsection{Case A}
\label{sec:caseA}
The first condition that we study is when the proxy $\hat{v}$ has both high conditional precision and conditional recall. In this case we can bound the error in the estimates without any additional assumptions.

\begin{theorem}
\label{thm:high-precision-and-recall}
If for any $\ell \in \{0,1\}$ the precision and recall of the proxy $\hat{v}$ is at least $1-\epsobs_A$, then we have that
$$
|G - \hat{G}| \leq 2 \cdot \epsobs_A.
$$
\end{theorem}

\begin{proof}
Below, we will bound $\delta_0$ and the analysis for $\delta_1$ will be identical. 

We can write $\delta_0$ as:
\begin{align}
    \label{eqn:main_equation}
    \begin{split}
    \delta_0 &=\Pr[y\!=\!1 | v\!=\!1, \ell\!=\!0] - \Pr[y\!=\!1 | \hat{v}\!=\!1, \ell\!=\!0] \\
    &= (p_0 - r_0) \Pr[y\!=\!1 | v\!=\!1, \hat{v}\!=\!1, \ell\!=\!0]\\
    &\;\;\;\; + r_0 \Pr[y\!=\!1 | v\!=\!1, \hat{v}\!=\!0, \ell\!=\!0]\\
    &\;\;\;\; -p_0 \Pr[y\!=\!1 | v\!=\!0, \hat{v}\!=\!1, \ell\!=\!0].
     \end{split}
\end{align}

Since the third term above is negative, and the probabilities are bounded by $1$, we can conclude that $\delta_0$  is at most $p_0$. Similarly, the negation, i.e., $-\delta_0$ is bounded by $r_0$. Hence we get that
\begin{align*}
    \Big| \Pr[y\!=\!1 | v\!=\!1, \ell\!=\!0] - \Pr[y\!=\!1 | \hat{v}\!=\!1, \ell\!=\!0] \Big| &\leq \max(p_0, r_0) \\
    &\leq \epsobs_A
\end{align*}
\end{proof}

Note that precision and recall have to be relatively high so that the bound above can be small enough. For instance, even if the proxy has precision and recall around $0.9$, the above bound will be around $0.2$, which would likely be too large for fairness evaluation in practical settings. The goal of the next two sections will be to progressively refine this bound by analyzing potential correlations of ($y, v, \hat{v}, \ell$).

\section{Alternate bounds based on other correlations.}
\label{sec:bound-graph}

In many cases requiring high precision and recall on the proxy may not be achievable. Can one still expect the estimation error to be small? We next characterize other conditions based on correlations between the different parameters~($y, v, \hat{v}, \ell$) that enable us to bound the error. We will study cases B1 and B2, that identify two different canceling effects under which the error might be small, despite not having high precision and recall. In order to understand these it would be beneficial to refer to the confusion matrix as shown in Table~\ref{tbl:confusion-matrix}.

\subsection{Intuition for the cases.}
Recall $\delta_\ell$ defined in Equation~\eqref{eqn:equation_delta} as
\begin{align*}
\begin{split}
\delta_{l} &= (p_l - r_l) \cdot (\Pr[y\!=\!1 | \hat{v}\!=\!1, v\!=\!1, \ell\!=\!l]) \\
         &\;\;\;\; + r_l \Pr[y\!=\!1 | \hat{v}\!=\!0, v\!=\!1, \ell\!=\!l] \\
         &\;\;\;\; - p_l \Pr[y\!=\!1 | \hat{v}\!=\!1, v\!=\!0, \ell\!=\!l]
 \end{split}
\end{align*} 

This equation above represents the estimation error made on the slice $\ell$. Additionally, the final estimation error is equal to the difference between the errors made on each slice:
\begin{align*}
|G - \hat{G}| = |\delta_1 - \delta_0|
\end{align*}

In the previous section, we bounded each term in $\delta_0$ and $\delta_1$ by relying on the performance of the proxy $\hat{v}$ (case A). Below we discuss alternate natural conditions that would suffice for good gap estimation:
\begin{itemize}
    \item [B1.] This case captures scenarios when the second and third term in $\delta_0$ (similarly for $\delta_1$) approximately cancel each other, indicating that the contribution of examples wrongly ignored~(false negative) is approximately equal to the examples wrongly included~(false positive). Additionally, the first term is small.
    \item [B2.] This case captures scenarios when $\delta_0$ and $\delta_1$ have similar values and therefore ``cancel each other,'' indicating that the correlations~(between $y, \hat{v}, v$) have a similar behavior for $\ell=0$ and $\ell=1$.
\end{itemize}
In the rest of the section, we will prove theorems that materialize the above intuition rigorously. We will decouple the effect of the the performance of the proxy~(precision and recall) on the estimation error from the properties of the confusion matrix.

\subsection{Case B1: Precision close to recall and closeness of diagonals}
\label{sec:caseB}
In this case we do not assume that the proxy has high precision and recall, but instead simply assume that the precisions and recalls~(conditioned on $\ell$) are close to each other, up to $\epsobs_B$. Clearly this assumption by itself is not enough to guarantee a small error. We also assume  that the diagonal entries in the confusion matrix of Table~\ref{tbl:confusion-matrix} are $\epsun_B$-close to each other. This is made precise in the definition below.

\begin{definition}[Closeness of Diagonals]
There exists a small $0 < \epsun_{B1} \leq 1$ such that
\begin{align*}
    \Big|\Pr[y\!=\!1 | v\!=\!1, \hat{v}\!=\!0, \ell\!=\!0] - \Pr[y\!=\!1 | v\!=\!0, \hat{v}\!=\!1, \ell\!=\!0]\Big|\\
    \leq \epsun_{B1} \\
    \Big|\Pr[y\!=\!1 | v\!=\!1, \hat{v}\!=\!0, \ell\!=\!1] - \Pr[y\!=\!1 | v\!=\!0, \hat{v}\!=\!1, \ell\!=\!1]\Big|\\
    \leq \epsun_{B1} 
    \end{align*}
\end{definition}
One can expect the above condition to hold in certain scenarios. Notice that for any $\ell$, the condition concerns the disagreement region, i.e., $\{v=1, \hat{v}=0\}$, and $\{v=0, \hat{v}=1\}$. In many cases this region could represent ``hard to predict'' points. The closeness on the diagonal is that the outcome $y$ has similar behavior for both errors types of $\hat{v}$. Returning to the example from Section~\ref{sec:intro} on predicting toxicity of online comments with the topic being the covariate~($v$), the classifier $\hat{v}$ may find it hard to predict certain esoteric topics well. It is then natural to expect that the toxicity prediction model may also perform poorly for these topics and its errors may be somewhat random, thereby making $\epsilon_{B_1}$ small. There may also be many cases where one can indeed not expect $\eps_{B_1}$ to be small. However, we will soon see in Section~\ref{sec:combining-ABC} a more refined bound that has a lower order dependence on such errors. Hence, even a weak constraint~(on $\eps_{B_1}$) suffices for getting smaller bounds.

Under the closeness of diagonal condition we have the following theorem.
\begin{theorem}
\label{thm:condition-B}
Let $\epsun_{B_1} \in [0,1]$ be the smallest constant such that the closeness of diagonal condition holds with $\epsun_{B1}$ and that $|r_0 - p_0|, |r_1 - p_1|$ are bounded by $\epsobs_{B1}$ then we have that
$$
|G - \hat{G}| \leq 2(\epsobs_{B1} + \epsun_{B1}).
$$
\end{theorem}

As Compared to Theorem~\ref{thm:high-precision-and-recall}, case B1 does not require an absolute bound on the precisions and recalls and quantifies the error in terms of their relative closeness~($\gamma_{B_1}$). This can be advantageous in some cases; the quantities captured by $\epsobs_{B1}$ might be smaller than $\epsobs_A$.

\begin{proof}
As before we will establish bounds for both $\delta_0$ and $\delta_1$. Let's consider $\delta_0$. From the analysis in the previous section we have that
\begin{align*}
\begin{split}
    &\Pr[y\!=\!1 | v\!=\!1, \ell\!=\!0] - \Pr[y\!=\!1 | \hat{v}\!=\!1, \ell\!=\!0] \\
    &= (p_0 - r_0) \Pr[y\!=\!1 | v\!=\!1, \hat{v}\!=\!1, \ell\!=\!0]\\
    &\;\;\;\; + r_0 \Pr[y\!=\!1 | v\!=\!1, \hat{v}\!=\!0, \ell\!=\!0] \\
    &\;\;\;\; -p_0 \Pr[y\!=\!1 | v\!=\!0, \hat{v}\!=\!1, \ell\!=\!0]\\
    &\leq (p_0 - r_0) \Pr[y\!=\!1 | v\!=\!1, \hat{v}\!=\!1, \ell\!=\!0]\\
    &\;\;\;\; + r_0 \Pr[y\!=\!1 | v\!=\!1, \hat{v}\!=\!0, \ell\!=\!0]\\
    &\;\;\;\; - p_0 r_0 \Pr[y\!=\!1 | v\!=\!1, \hat{v}\!=\!0, \ell\!=\!0] + \epsun r_0\\
    &\leq |p_0 - r_0|  + \epsun r_0 \leq \epsobs_{B1} + \epsun_{B1}.
\end{split}
\end{align*}

A similar analysis establishes that $\delta_1 \leq \epsobs_{B1} + \epsun_{B1}$.
\end{proof}

\subsection{Case B2: Closeness of precisions/recalls and closeness of models}
\label{sec:caseC}

The previous case identified conditions under which both $\delta_0$ and $\delta_1$ achieve small values, leading to a small estimation error. This error depends on the closeness of precisions and recalls within each slice. However, in practice this may not always hold and one may prefer an alternate condition. We next identify another natural condition that allows for precision and recall within each slice to be arbitrary, and both $\delta_0$ and $\delta_1$ could be large. However, one may still expect the estimation error~($|G - \hat{G}|$) to be small if the errors in $\delta_0$ and $\delta_1$ have a cancelling effect. We present such a condition next.

\begin{definition}[Model Closeness]
\label{def:model-closeness}
There exists $g > 0$ and a $0 < \epsun_{B2} \leq 1$ such that for all $b,c \in \{0,1\}$,
\begin{align*}
&\Pr[y\!=\!1 | v\!=\!b, \hat{v}\!=\!c, \ell\!=\!1] \\
&= \Pr[y\!=\!1 | v\!=\!b, \hat{v}\!=\!c, \ell\!=\!0] + g \pm \epsun_{B2}.
\end{align*}
\end{definition}

In terms of the confusion matrix defined in Table~\ref{tbl:confusion-matrix} the above condition requires that the entries of the confusion matrix for $\ell=1$ are noisy translations of the entries for $\ell=0$, $g$ representing here the gap between groups. 

Notice that case B2 allows for arbitrary behavior within each slice~($\ell$). Hence, for the estimation error to be small, some assumption is needed on the correlations across slices and the model closeness condition in Definition~\ref{def:model-closeness} captures that. In practice one may expect the condition to hold if the proxy $\hat{v}$ introduces similar errors across slices. As an example, in the context of the toxicity prediction example, consider a proxy~($\hat{v}$) that was (unintentionally) trained on a corpus consisting of only one group, say $\ell=1$. When applying $\hat{v}$ to the group $\ell=0$ one can expect $\hat{v}$ to introduce systematic errors~(captured by $g$ in the definition).

In this case we have the following theorem
\begin{theorem}
\label{thm:condition-C}
Let $\epsobs_{B2}, \epsun_{B2} \in [0,1]$ and $g$ be such that the model closeness holds with $\epsun_{B2}$ and that $|r_0 - r_1|, |p_0 - p_1|$ are bounded by $\epsobs_{B2}$. Then we have that
$$
|G - \hat{G}| \leq 2 \cdot \epsobs_{B2} + 3 \cdot \epsun_{B2}. 
$$
\end{theorem}

Note that in the above Theorem we do not require precisions and recalls within any slice to be high.

\begin{proof}
We would not be able to separately analyze $\delta_0$ and $\delta_1$ since both could be high. Instead we argue about the difference. We have
\begin{align}
    \delta_0 - \delta_1 &= T_1 + T_2 + T_3,
\end{align}
where
\begin{align*}
    T_1 &= (p_0 - r_0) \Pr[y\!=\!1 | v\!=\!1, \hat{v}\!=\!1, \ell\!=\!0]\\
    &\;\;\;\;- (p_1 - r_1) \Pr[y\!=\!1 | v\!=\!1, \hat{v}\!=\!1, \ell\!=\!1],
\end{align*}
\begin{align*}
    T_2 &= r_0 \Pr[y\!=\!1 | v\!=\!1, \hat{v}\!=\!0, \ell\!=\!0]\\
    &\;\;\;\;- r_1 \Pr[y\!=\!1 | v\!=\!1, \hat{v}\!=\!0, \ell\!=\!1],
\end{align*}
and
\begin{align*}
    T_3 &= p_1 \Pr[y\!=\!1 | v\!=\!0, \hat{v}\!=\!1, \ell\!=\!1]\\
    &\;\;\;\;- p_0 \Pr[y\!=\!1 | v\!=\!0, \hat{v}\!=\!1, \ell\!=\!0].
\end{align*}

Since $|r_0 - r_1|, |p_0 - p_1|$ are both bounded by $\epsobs_{B2}$ and model closeness holds we have
\begin{align*}
    T_1 &= (p_0 - r_0 - p_1 + r_1) \Pr[y\!=\!1 | v\!=\!1, \hat{v}\!=\!1, \ell\!=\!0] \\
    &\;\;\;\;- (p_1 - r_1) g \pm \epsun_{B2},
\end{align*}
\begin{align*}
    T_2 &= (r_0 - r_1) \Pr[y\!=\!1 | v\!=\!1, \hat{v}\!=\!0, \ell\!=\!0]\\
    &\;\;\;\;- r_1 g \pm \epsun_{B2},
\end{align*}
and
\begin{align*}
    T_3 &= (p_1-p_0) \Pr[y\!=\!1 | v\!=\!0, \hat{v}\!=\!1, \ell\!=\!0]\\
    &\;\;\;\;+ p_1 g \pm \epsun_{B2}
\end{align*}

Adding the above we get that $T_1 + T_2 + T_3 \leq 2 \cdot \epsobs_{B2} + 3 \cdot \epsun_{B2}$. Similarly, the negation $\delta_1 - \delta_0$ can be bounded.
\end{proof}

\section{Synthesis: A refined bound.}
\label{sec:combining-ABC}

So far we have presented three natural conditions under which the error in the estimation can be bounded. We recap them below as:
\begin{itemize}
    \item [A.] The first bound uses only the performance of the proxy $\hat{v}$, therefore needs high precision and recall.
    \item [B1.] The second bound depends on the precision being close to the recall, in addition to the closeness of diagonals.
    \item [B2.] The third bound depends on the closeness of precision and recalls between slices and the relationship between the confusion matrices of the two slices.
\end{itemize}

The bounds earlier concern each of the above conditions in isolation. As a result, they are likely to be wide except unless one of the constraint is very strong, which might not be realistic. In this section we study whether one can combine the power of the three conditions to refine the error bound that naturally adapts to the extent to which the above conditions are satisfied. Towards that end we have the following theorem.

\begin{theorem}
\label{thm:combine-ABC}
Let $\epsobs_A$, ($\epsobs_{B1}, \epsun_{B1}$) and ($\epsobs_{B2}, \epsun_{B2}$) be the errors up to which conditions A, $B1$ and $B2$ above hold. Then we have that 

\begin{align*}
|G-\hat{G}| & \leq 2 \min(\epsobs_A, \epsobs_{B1}, \epsobs_{B2}) \\
                  &+ \epsun_{B2} \cdot(2\epsobs_{A} + \epsobs_{B1}) + \epsun_{B1} \cdot \epsobs_{B1}
\end{align*}
\end{theorem}

This theorem combines the three conditions highlighted before and achieves a refined bound. In particular, the dependence on the parameters governing the precision/recall of the classifier~($\gamma$) is linear, whereas in contrast to earlier theorems, the errors due to other correlations~($\epsilon$) have a multiplying effect. Hence, in practice even a weak estimate of the parameters $\epsilon_{B_1}$ and $\epsilon_{B_2}$ may suffice to get an estimate of how good $\hat{G}$ is. 

\vspace*{2mm}

\begin{proof}
Recalling the definitions of $T_1, T_2$ and $T_3$ from the proof of Theorem~\ref{thm:condition-C} we get that
\begin{align*}
    T_1 &= (p_0 - r_0 - p_1 + r_1) \Pr[y\!=\!1 | v\!=\!1, \hat{v}\!=\!1, \ell\!=\!0]\\
    &\;\;\;\;- (p_1 - r_1) g \pm \epsun_{B2} \cdot (p_1 - r_1),
\end{align*}
\begin{align*}
    T_2 &= (r_0 - r_1) \Pr[y\!=\!1 | v\!=\!1, \hat{v}\!=\!0, \ell\!=\!0]\\
    &\;\;\;\;- r_1 g \pm \epsun_{B2}\cdot r_1,
\end{align*}
and
\begin{align*}
    T_3 &= (p_1-p_0) \Pr[y\!=\!1 | v\!=\!0, \hat{v}\!=\!1, \ell\!=\!0]\\
    &\;\;\;\; + p_1 g \pm \epsun_{B2} \cdot p_1.
\end{align*}
Adding up we get that
\begin{align*}
    &T_1 + T_2 + T_3 \\
    &\leq (p_0 - r_0 - p_1 + r_1) \Pr[y\!=\!1 | v\!=\!1, \hat{v}\!=\!1, \ell\!=\!0]\\
    &\;\;\;\; + (r_0 - r_1) \Pr[y\!=\!1 | v\!=\!1, \hat{v}\!=\!0, \ell\!=\!0] \\
    &\;\;\;\; + (p_1-p_0) \Pr[y\!=\!1 | v\!=\!0, \hat{v}\!=\!1, \ell\!=\!0]\\
    &\;\;\;\; + \epsun_{B2} \cdot (2\epsobs{A}+ \epsobs_{B1})\\
    &\leq (p_0 - r_0 - p_1 + r_1) \Pr[y\!=\!1 | v\!=\!1, \hat{v}\!=\!1, \ell\!=\!0]\\
    &\;\;\;\; + (r_0 - r_1) \Pr[y\!=\!1 | v\!=\!1, \hat{v}\!=\!0, \ell\!=\!0] \\
    &\;\;\;\; + (p_1-p_0) \Pr[y\!=\!1 | v\!=\!1, \hat{v}\!=\!0, \ell\!=\!0] + \epsun_{B2} (2 \epsobs_{A} + \epsobs_{B1}) \\
    &\;\;\;\;+\epsobs_{B1} \cdot \epsun_{B1}\\
    &\leq (p_0 - r_0 - p_1 + r_1) + \epsun_{B2} \cdot(2 \epsobs_{A}+ \epsobs_{B1}) \\
    &\;\;\;\;+\epsobs_{B1} \cdot \epsun_{B1}.
\end{align*}
The bound follows from noticing that
$$
p_0 - r_0 - p_1 + r_1 \leq 2 \min(\epsobs_{A},\epsobs_{B1},\epsobs_{B2}).
$$
\end{proof}
Compared to Theorem~\ref{thm:high-precision-and-recall}, the first term in the bound above may be significantly smaller than $2\epsobs_A$. Furthermore, in contrast to Theorem~\ref{thm:condition-B} and Theorem~\ref{thm:condition-C}, the dependence on $\epsilon_{B_1}$ and $\epsilon_{B_2}$ is diminished due to multiplication by other error terms. Due to this cancelling effect and the linear dependency on the the parameters related to the proxy performance~($\epsobs$), our theory recommends that the focus of a practitioner should then be on minimizing the first term. We next validate our theory via experiments on simulated data.

\section{Simulations}    
\label{sec:experiments}
We proposed in Section~\ref{sec:bound-classifier} an initial bound based solely on the classifier performance and then highlighted that some constraints over the confusion matrix~(captured by $\epsun$) lead to errors canceling each other. 

In this section, we use simulations to further validate our theory. The findings of this section are summarized below.
\begin{itemize}
\item We will observe that our theoretical bounds from Section~\ref{sec:caseA} that are based solely on the classifier performance represent worst-case scenarios, and that the gap in most cases is significantly smaller than our upper bounds. In particular, the $95$th percentile of the distribution of errors is typically more than $2\times$ smaller.

\item We will demonstrate the importance of making {\em some} assumptions on the error parameters via $\epsun$, instead of simply relying on classifier performance via $\epsobs$. The constraint on the parameter values~($\epsilon$) rules out some of these worst-case scenarios.

\item Finally, we will show that even weak constraints over the parameters~(with $\epsun$) are sufficient to significantly reduce the bound.
\end{itemize}

\subsection{Simulated estimation errors are often smaller than the theoretical bounds.}

\begin{figure}
  \centering
  \includegraphics[width=1.1\linewidth]{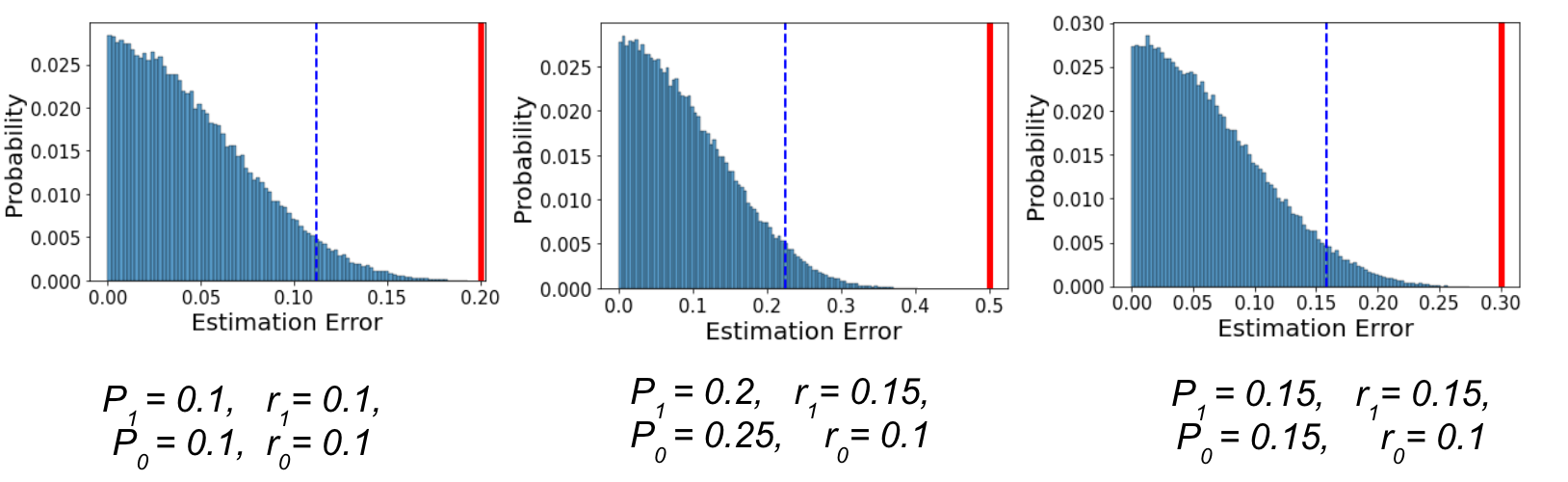}
  \caption{Simulated distribution of estimation errors. Each graph represent a different choice of the precision and recall of the proxy. The dashed blue line represent $95$th-percentiles and the unbroken red one is the theoretical bound from Theorem~\ref{thm:high-precision-and-recall}.
  }
  \label{graphA}
  \vspace{-4mm}
\end{figure}

In order to compute the estimation error, we rely on Equation~\eqref{eqn:main_equation} where we need to define both the classifier performance and the confusion matrix from Table~\ref{tbl:confusion-matrix}. We first illustrate the scenario where we only control for classifier performance~(case A) and therefore do not enforce any constraints on the values in the confusion matrix from Table~\ref{tbl:confusion-matrix}.

\noindent \textbf{Setup:} 
We first set classifier-related quantities ($p_1$, $r_1$, $p_0$ and $r_0$). Note that we display results for a few choice of these parameters, however the observations hold across any choices made for those quantities.

Once we set the above quantities,  we generate the confusion matrix, aka the six parameters, $\Pr[y\!=\!1 | v\!=\!i, \hat{v}\!=\!j, \ell\!=\!l]$, for $i, j, l \in \{0,1\}$, by sampling from a uniform prior in [0,1], $U[0,1]$. This reflects the situation where no prior structure is made on the confusion matrix. Note that $\Pr[y\!=\!1 | v\!=\!0, \hat{v}\!=\!0, \ell\!=\!l]$ are not needed. 
\\

\noindent \textbf{Results:}
The results are displayed in Figure~\ref{graphA}. First, we confirm that the bound provided by Theorem~\ref{thm:high-precision-and-recall} is correct. More importantly, we observe that these bounds reflect worst-case scenarios and that the errors are typically significantly smaller: indeed, the $95$th percentile is approximately $2\times$ smaller than the upper bound. This phenomena is due to the fact that the different terms in Equation~\eqref{eqn:main_equation} cancel each other leading to significantly lower errors than worst-case scenario.

This above setup uses as a uniform prior over parameters and can lead to unrealistic scenarios that might be responsible for worst-case estimation errors. In the next sections will focus on more realistic scenarios by restricting the range of values of the parameters following the principles highlighted by our theory. This essentially constraints the possible values that the confusion matrix in Table~\ref{tbl:confusion-matrix} can take.

\subsection{Importance of assuming structure over the confusion matrix via $\epsun$.} 

In this section, we assume that we have a classifier with relatively good precision and recall (for instance $r_i, p_i \leq$ 0.1). The bound from \ref{thm:high-precision-and-recall} is still too large for a good estimate of the error~(in this case close to 0.2), and we show that assuming some structure over the correlations is key to getting better estimates.

\noindent \textbf{Setup:} To simulate more realistic scenarios, we now enforce some level of structure controlled by $\epsun$. As before, we first select a well performing classifier ($p_1= 0.07 $, $r_1= 0.09$, $p_0=0.05$ and $r_0=0.1$). We will report experiments only for this choice but the results are consistent across different choices. Then, we set a value $\epsun_{B1}$ and $\epsun_{B2}$, which reflects to what extent the contraints B1~(closeness of diagonals) and B2~(model closeness) hold. For instance, $\epsun_{B1}=0$ means equal parameters and $\epsun_{B1}=1$ is equivalent to having no structure.

We sample the parameters as follows:
\begin{itemize}
    \item $g \sim U[-1, 1]$, where $U$ denotes the uniform distribution.
    \item $\Pr[y\!=\!1 | v\!=\!1, \hat{v}\!=\!1, \ell\!=\!l] \sim U[0, 1-g]$, if $g \geq 0$, else $\sim U[|g|, 1]$. 
    \item $\Pr[y\!=\!1 | v\!=\!1, \hat{v}\!=\!0, \ell\!=\!l] \sim U[0, 1-g]$ if $g \geq 0$, else $\sim U[|g|, 1]$.
    \item $ \Pr[y\!=\!1 | v\!=\!0, \hat{v}\!=\!1, \ell\!=\!l] = \Pr[y\!=\!1 | v\!=\!1, \hat{v}\!=\!0, \ell\!=\!l] + U[-\epsilon_{B2}, \epsilon_{B2}]$.
    \item $\Pr[y\!=\!1 | v\!=\!i, \hat{v}\!=\!j, \ell\!=\!l] = \Pr[y\!=\!1 | v\!=\!i, \hat{v}\!=\!j, \ell\!=\!l] + g + U[-\epsilon_{B2}, + \epsilon_{B2}]$.
\end{itemize}
Note that the above sampling might not lead to valid probabilities (outside of [0,1]). If that happens, we restart again and sample all the parameters. This method ensures sampling the confusion matrix in Table~\ref{tbl:confusion-matrix} that satisfies our constraints.

We compute both the true and estimated gap~($G$ and $\hat{G}$) and report the distribution of errors~($error = | G - \hat{G} | $) for $N=100000$ independent runs.
\\

\begin{figure}[htbp]
  \centering
  \includegraphics[width=1.1\linewidth]{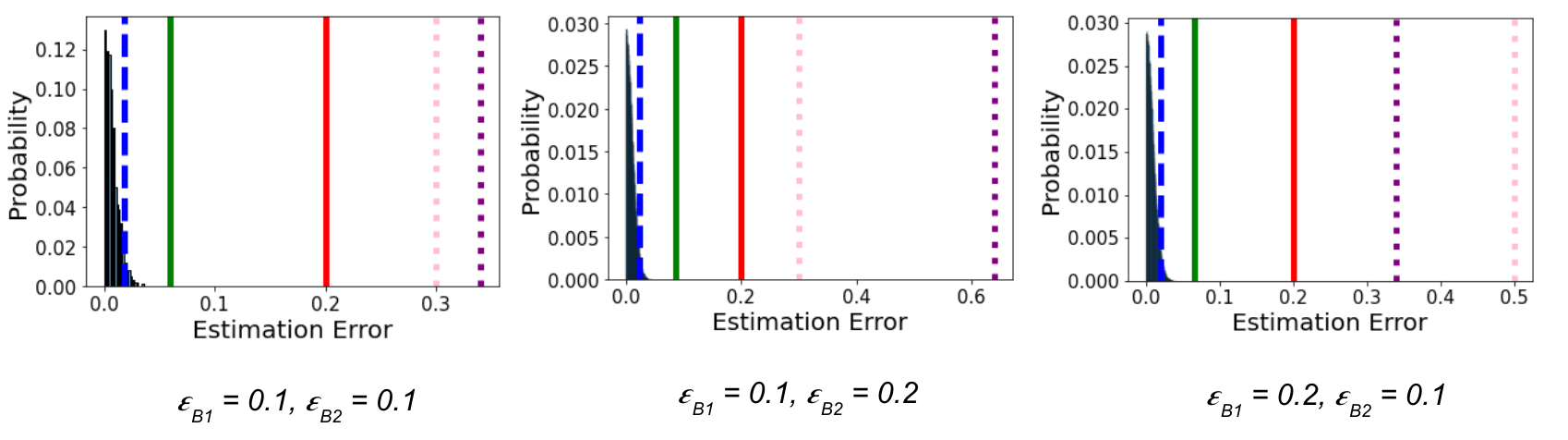}
  \caption{Refined bound from Theorem \ref{thm:combine-ABC} dominates other bounds. Each graph represent a different choice of ($\epsun_{B1}, \epsun_{B2}$). The dashed blue line represents the $95$th percentile of the simulated errors. The red (solid) line comes from Theorem~\ref{thm:high-precision-and-recall}. The pink~(resp. purple) line comes from Theorem~\ref{thm:condition-B}~(resp Theorem~\ref{thm:condition-C}). Finally, the green line is given by the refined bound from Theorem~\ref{thm:combine-ABC}.}
  \label{graph2}
  \vspace{-4mm}
\end{figure}

\noindent \textbf{Results:} We report in Figure~\ref{graph2} the histogram of the errors, as well as the various theoretical bounds. The dashed blue line represents the $95$th percentile of the simulated errors. The red (solid) line comes from Theorem~\ref{thm:high-precision-and-recall} and represents the case when the practitioner does not make any assumption on the structure of the confusion matrix. The pink~(resp. purple) line comes from Theorem~\ref{thm:condition-B}~(resp Theorem~\ref{thm:condition-C}) which focuses on one constraint on the confusion matrix. Finally, the green line is given by the refined bound from Theorem~\ref{thm:combine-ABC}.

We observe that the bounds from Case A, B1 or B2 are all quite wide~(typically above $0.2$ in this scenario). Note that their ordering in the plots may differ depending on how we set each constraint. More importantly, the bound from Theorem \ref{thm:combine-ABC} is significantly better than any of the individual ones thereby illustrating that by considering all the different sources of error, we are able to compute better estimates. In particular, it is important to assume some structure over the confusion matrix, through $\epsun$, in order to prevent worst-case scenarios.

\subsection{Even weak assumptions on $\epsun$ are enough to significantly reduce the bound.}
The previous experiment highlights the importance of adding some constraints on potential values of the confusion matrix from Table~\ref{tbl:confusion-matrix}. However it might be hard to assume strong conditions~(i.e., $\epsun$ close to $0$). In this section, we demonstrate that even loose estimates are enough to significantly improve our estimates.

\begin{figure}[htbp]
  \centering
  \includegraphics[width=1.0\linewidth]{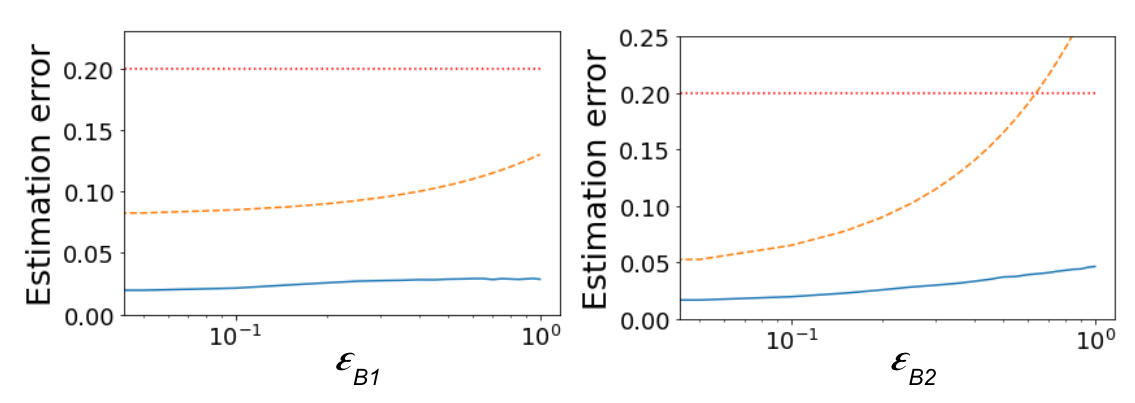}
  \caption{On the left (resp. right) plot, we set $\epsun_{B2}$ (resp. $\epsun_{B1}$) and vary $\epsun_{B1}$ (resp. $\epsun_{B2}$). The blue lines represent the $95$th-percentile of the estimation errors. The refined bound from Theorem~\ref{thm:combine-ABC} (dashed orange) dominates the one from Theorem~\ref{thm:high-precision-and-recall} (dotted red), even with weak constraints ($\epsun$ does not have to be small).}
  \label{graph3}
  \vspace{-4mm}
\end{figure}

\textbf{Setup:}
To that purpose, we study how the estimation errors evolve as we progressively strengthen the condition on the correlations. More precisely, we take a good classifier with $p_1= 0.07 $, $r_1= 0.09$, $p_0=0.05$ and $r_0=0.1$~(arbitrary values for sake of simplicity). We then set one of~($\epsun_{B1}$, $\epsun_{B2}$) to some fixed value~(here $0.2$) and vary the other one from $0$~(very strong condition) to $1$~(no condition). Then we use the same approach as in the previous experiment to sample all the remaining variables. We finally analyze the resulting estimation errors, as well as theoretical bounds.

\noindent \textbf{Results:} Results are displayed in Figure~\ref{graph3}. The blue lines represent the $95$th-percentile of the estimation errors from the simulation, the dashed orange line is the theoretical bound from Theorem~\ref{thm:combine-ABC} and the dotted red one from case A, i.e., Theorem~\ref{thm:high-precision-and-recall}. 

We can see that, even in this case where the covariate classifier is relatively good ($r_i, p_i \leq 0.1$), the bound from Theorem~\ref{thm:high-precision-and-recall} remains quite large $(0.2)$. However the refined bound in Theorem~\ref{thm:combine-ABC} leads to a tighter estimate ~(except in the case when $\epsun_{B2}$ becomes close to $1$ in the right plot). Furthermore, both the theoretical guarantees and empirical results significantly improve when the enforced conditions become quantitatively stronger.
More importantly, these graphs demonstrate that Theorem~\ref{thm:combine-ABC} does not require strong constraints to achieve a satisfying bound. For instance, the left graph shows that the theoretical bound is still lower than $0.1$ if $\epsun_{B2} \leq 0.4$~(a relatively weak condition). Additionally the $95$th percentile is between $.02$ and $.03$, which shows even better estimate in practice.

This demonstrates that {\em some} assumptions on $\epsun$, even if weak, is key in getting good estimates of the error.

\section{Conclusion}
\label{sec:conclusions}
Measuring the fairness of a model might often be restricted by the data available. We analyze in this paper the common scenario of estimating a fairness metric via a proxy covariate. We identify various sources of errors that affect the use of the proxy covariate in measuring model fairness. Our theory  demonstrates that the errors can be driven down not only by the performance of the proxy such as precision and recall, but also by the underlying correlations in the data distribution. 
We provided a refined analysis that combines the various sources of errors and highlights how these errors interact to affect the final bound. More importantly, our work shows that one does not need to make strong independence assumptions in order to study and obtain bounds for the problem of measuring model fairness via proxy covariates.

Finally we demonstrate through simulations that exploiting structure in the correlations, even if loose, is the key to obtaining better bounds on the error in measuring the fairness metrics. 
We believe that our work will guide practitioners towards making informed choices when designing proxy covariates for fairness evaluation, or obtaining bounds on the errors in their estimation procedures.

\section{Acknowledgements:} The authors would like to thank Alexander D'Amour for his valuable feedback on this paper.


\appendix

\section{Proofs from Section~\ref{sec:simple-cases}}

\begin{proof}[Proof of Theorem~\ref{thm:strong-independence}]
Define, 
\begin{align*}
    \label{def:delta}
    \delta_0 & \coloneqq \Pr[y\!=\!1 | v\!=\!1, \ell\!=\!0] - \Pr[y\!=\!1 | \hat{v}\!=\!1, \ell\!=\!0]\\
    \delta_1 & \coloneqq \Pr[y\!=\!1 | v\!=\!1, \ell\!=\!1] - \Pr[y\!=\!1 | \hat{v}\!=\!1, \ell\!=\!1].
\end{align*}
Under the independence assumption it is easy to see that
\begin{align*}
    \Pr[y\!=\!1 | v\!=\!1,  \ell\!=\!0] = \Pr[y\!=\!1 | \hat{v}\!=\!1, \ell\!=\!0].
\end{align*}
This implies that $\delta_0=0$, and similarly that $\delta_1=0$. Hence in this case the error in our estimates is zero. 
\end{proof}

\begin{proof}[Proof of Theorem~\ref{thm:high-precision}]
We will show that both $\delta_0$ and $\delta_1$ are bounded in magnitude by at most $p$ thereby establishing the theorem. We next show how to bound $\delta_0$. The analysis for $\delta_1$ will be identical. Recall that 
$$
\delta_0 = \Pr[y\!=\!1 | v\!=\!1, \ell\!=\!0] - \Pr[y\!=\!1 | \hat{v}\!=\!1, \ell\!=\!0].
$$
Under the independence assumption we have
\begin{align*}
\delta_0 &= \Pr[y\!=\!1 | v\!=\!1, \ell\!=\!0]\\
&\;\;\;\;- \Pr[y\!=\!1 | v\!=\!1, \ell\!=\!0]\Pr[v\!=\!1 | \hat{v}\!=\!1, \ell\!=\!0]\\
&\;\;\;\;- \Pr[y\!=\!1 | v\!=\!0, \ell\!=\!0]\Pr[v\!=\!0 | \hat{v}\!=\!1, \ell\!=\!0]\\
&\leq \Pr[y\!=\!1 | v\!=\!1, \ell\!=\!0]\\
&\;\;\;\;- \Pr[y\!=\!1 | v\!=\!1, \ell\!=\!0]\Pr[v\!=\!1 | \hat{v}\!=\!1, \ell\!=\!0]\\
&= \Pr[y\!=\!1 | v\!=\!1, \ell\!=\!0] \big(1 - \Pr[v\!=\!1 | \hat{v}\!=\!1, \ell\!=\!0] \big)\\
&\leq r
\end{align*}
Similarly, we get that
\begin{align*}
-\delta_0 &=  \Pr[y\!=\!1 | v\!=\!1, \ell\!=\!0]\Pr[v\!=\!1 | \hat{v}\!=\!1, \ell\!=\!0]\\
&\;\;\;\;+ \Pr[y\!=\!1 | v\!=\!0, \ell\!=\!0]\Pr[v\!=\!0 | \hat{v}\!=\!1, \ell\!=\!0]\\
&\;\;\;\;- \Pr[y\!=\!1 | v\!=\!1, \ell\!=\!0]\\
&= \Pr[y\!=\!1 | v\!=\!1, \ell\!=\!0] \big(\Pr[v\!=\!1 | \hat{v}\!=\!1, \ell\!=\!0]-1 \big)\\
&\;\;\;\;+ \Pr[y\!=\!1 | v\!=\!0, \ell\!=\!0]\Pr[v\!=\!0 | \hat{v}\!=\!1, \ell\!=\!0]\\
&= \Pr[y\!=\!1 | v\!=\!1, \ell\!=\!0] \Pr[v\!=\!0 | \hat{v}\!=\!1, \ell\!=\!0]\\
&\;\;\;\;+ \Pr[y\!=\!1 | v\!=\!0, \ell\!=\!0]\Pr[v\!=\!0 | \hat{v}\!=\!1, \ell\!=\!0]\\
&= \Pr[v\!=\!0 | \hat{v}\!=\!1, \ell\!=\!0]\\
&\leq r.
\end{align*}
\end{proof}
\begin{proof}[Proof of Theorem~\ref{thm:high-recall}]
The proof is exactly identical to the proof of Theorem~\ref{thm:high-precision} above by simply switching the role of $v$ and $\hat{v}$.
\end{proof}

\bibliographystyle{ACM-Reference-Format}
\bibliography{main}

\end{document}